\documentclass{ifacconf}
\usepackage{amssymb}
\usepackage{comment}
\usepackage{amsmath}
\let\theoremstyle\relax
\usepackage{amsthm} 
\usepackage{graphicx}      
\usepackage{natbib}    
\usepackage{subcaption} 
\usepackage{balance} 

\usepackage{float} 
\usepackage[inline]{enumitem}
\newtheorem{innercustomgeneric}{\customgenericname}
\providecommand{\customgenericname}{}
\newcommand{\newcustomtheorem}[2]{%
  \newenvironment{#1}[1]
  {%
   \renewcommand\customgenericname{#2}%
   \renewcommand\theinnercustomgeneric{##1}%
   \innercustomgeneric
  }
  {\endinnercustomgeneric}
}

\newcustomtheorem{customthm}{Theorem}
\newcustomtheorem{customlemma}{Lemma}
\newcustomtheorem{customprop}{Proposition}

\theoremstyle{plain} 
\newtheorem{theorem}{Theorem}
\newtheorem{lemma}[theorem]{Lemma}
\newtheorem{definition}[theorem]{Definition}
\newtheorem{remark}{Remark}
\newtheorem{proposition}[theorem]{Proposition}
\usepackage[dvipsnames]{xcolor}
\newcommand{\remove}[1]{}
\begin{document}
\begin{frontmatter}

\title{Spatio-Temporal Reconnection for Multi-Robot Networks using Adaptive Prescribed-Time CBFs}

\thanks[footnoteinfo]{This work was supported in part by the U.S. National Science Foundation under Award 2528997.}

\author[First]{Hao Liu} 
\author[Second]{Yupeng Yang} 
\author[First]{Yanze Zhang}
\author[First]{Wenhao Luo}

\address[First]{Department of Computer Science, University of Illinois Chicago, Chicago, IL 60607, USA (e-mail: hliu232,yzhan361,wenhao@uic.edu).}
\address[Second]{Department of Computer Science, University of North Carolina at Charlotte, Charlotte, NC 28223, USA (e-mail: yyang52@charlotte.edu)}

\begin{abstract}               
In multi-robot systems, maintaining \textit{persistent} communication graph connectivity is often overly restrictive, especially when robots have limited communication ranges but operate in large environments. Instead, allowing robots to temporarily disconnect and later reconnect is often more desirable for efficient task execution while still ensuring timely information sharing across the team. In this paper, we propose an \textit{adaptive prescribed-time control barrier function} (adaptive PT-CBF) framework that enables robots to temporarily disconnect and re-enter the communication range within an adjustable and feasible prescribed time. Moreover, we introduce a reconnection triggering mechanism that jointly considers task execution and reconnection urgency, thereby providing a principled way to decide when reconnection should occur. Theoretical analysis justifies convergence to the satisfying reconnection within a prescribed finite time. Experimental results validate the performance of our proposed adaptive PT-CBF with improved task efficiency and satisfying reconnections. 
\end{abstract}

\end{frontmatter}
\section{Introduction}
Networked multi-robot systems often rely on wireless communication to coordinate tasks such as formation control,
area coverage,
and environment estimation.
In this setting, \emph{global connectivity} is often required where all robots remain in a single connected component, so that information from any robot can reach all others through available communication links \citep{sabattini2013distributed}.

Due to limited communication ranges, however, constraining robots' motion to enforce pairwise connections within range at all times can be overly restrictive, and may hinder robots from effectively performing their primary tasks. Existing methods \citep{luo2020behavior, yang2024integrating, yang24decentralized} address this issue by only maintaining a minimum number of edges required for global connectivity, thereby preserving communication with flexible task execution. When temporary disconnection is acceptable, such persistent connectivity constraints may impose unnecessary motion restrictions and cause robots to deviate from their nominal task trajectories.  
In practice, communication requirements vary across tasks. While critical information such as robot states and control signals requires continuous sharing, high-bandwidth data (e.g., images and videos) may only need \textit{periodic synchronization} among robots.
This motivates a relaxed formulation of connectivity maintenance. 
Under this setting, robots should be allowed to temporarily break global connectivity, disperse to different regions to carry out their individual tasks (e.g., collect local information), and return to within communication range at different times for periodic information exchange. This perspective highlights the importance of \emph{reconnection}, namely, the ability of a multi-robot team to progressively re-establish a connected communication topology even if the network becomes temporarily disconnected.

While not originally developed for controlling robots' motion for reconnection,
methods such as Finite-time Control Barrier Function (FT-CBF)~\citep{li2018formally} provide a useful tool 
for enforcing finite-time convergence, with an upper bound on the convergence time determined by the initial system state. However, an exact convergence time cannot be pre-assigned by the user, which limits its applicability in scenarios where timely reconnection is critical. To design such a controller in terms of driving the robot back to a desired set with a prespecified time,
Prescribed-time Control Barrier Functions (PT-CBF)~\citep{abel2023prescribed} have been proposed to introduce time-varying gains that grow unbounded as the system approaches the prescribed time limit, thus enforcing convergence within the desired time.
However, in the context of motion control for reconnection, it remains challenging to decide when the reconnection should occur within a mission, beyond timely system convergence (i.e., driving robots to within communication range) once activated.

Due to these limitations, a separate line of work addresses task-level scheduling of communication configurations \citep{wang2025collision, ozkahraman2020combining}, 
but these approaches generally do not provide \emph{formal} guarantees that the required communication conditions can actually be realized within a desired time, particularly when robots initially lie outside the communication range. As a result, scheduled connectivity patterns may be difficult to realize in practice when robot dynamics, control limits, and task objectives interfere with required motion reconfiguration.

In this paper, we propose an approach that allows robots to temporarily relax connectivity constraints while guaranteeing reconnection within a prescribed time window. 
The contributions are summarized as follows: 
1) we introduce an adaptive prescribed-time CBF that combines the state-dependent convergence bound of the finite-time Control Barrier Function (FT-CBF) with the prescribed-time structure of the Prescribed-Time Control Barrier Function (PT-CBF), enabling automatic selection of the prescribed time while guaranteeing reachability before the deadline, 2) we propose a spatio-temporal reconnection mechanism that adaptively triggers pairwise reconnection constraints for efficient multi-robot coordination, and 3) we provide theoretical analysis and experiments validating the effectiveness and scalability of the proposed approach.

\section{Preliminaries}

We consider a team of $N$ mobile robots operating in a $d$-dimensional workspace. The dynamics of each robot $i \in \{1, \ldots, N\}$ are denoted by:
\begin{equation}\label{eq:dynamics}
    \dot{\mathbf{x}}_i = f_i(\mathbf{x}_i)+g_i(\mathbf{x}_i)\mathbf{u}_i,
\end{equation} 
where $f_i:\mathbb{R}^{d}\mapsto\mathbb{R}^{d}$ and $g_i:\mathbb{R}^{d}\mapsto\mathbb{R}^{d\times q}$ are locally Lipschitz continuous. $\mathbf{u}_i\in\mathbb{R}^{q}$ is the control input.

\subsection{Collision Avoidance}
Let $\mathbf{x}=\{\mathbf{x}_1,\ldots,\mathbf{x}_N\}\in\mathbb{R}^{dN}$ denote the joint robot state, and let $R_\mathrm{s}>0$ be the minimum inter-robot safety distance. 
For each robot pair $i$ and $j$, safety function is defined as
\begin{equation}
    h^\mathrm{s}_{i,j}(\mathbf{x})
    =
    \|\mathbf{x}_i-\mathbf{x}_j\|^2-R_\mathrm{s}^2,
    \quad \forall i>j.
    \label{eq:safe_function}
\end{equation}
The corresponding pairwise safe set can be defined by 
$h^\mathrm{s}_{i,j}(\mathbf{x})\ge 0$. By combining these pairwise safety sets, the overall safe set of the robotic team is expressed as
\begin{equation}
\mathcal{H}^\mathrm{s}
=
\bigcap_{i>j}
\{\mathbf{x}\in\mathbb{R}^{dN}:h^\mathrm{s}_{i,j}(\mathbf{x})\ge0\}.
\label{eq:overall_safe_sets}
\end{equation}
Using the standard Control Barrier Function (CBF) condition~\citep{ames2019control}, 
the time derivative of each safety function is constrained by
$\dot h(\mathbf{x},\mathbf{u})+\mathcal{K}(h(\mathbf{x}))\ge0$, 
which prevents $h(\mathbf{x})$ from becoming negative once it is initially nonnegative.
Thus, the admissible control set that renders $\mathcal{H}^\mathrm{s}$ forward invariant is given by\footnote{In the rest of this paper, we select $\mathcal{K}(h(\cdot))=\gamma h(\cdot)$ with $\gamma$ as a user-defined parameter~\citep{ames2019control}.}
\begin{equation}
\mathcal{B}^\mathrm{s}(\mathbf{x})
=
\left\{
\mathbf{u}\in\mathbb{R}^{qN}:
\dot{h}^\mathrm{s}_{i,j}(\mathbf{x},\mathbf{u})
+
\gamma h^\mathrm{s}_{i,j}(\mathbf{x})
\ge 0,\ \forall i>j
\right\}. \notag
\label{eq:cbf_safe}
\end{equation}
\subsection{Joint Global Connectivity}
\label{sec:joint-twgc}

In this paper, we model inter-robot communication as a proximity-based undirected graph $\mathcal{G}(t)=(\mathcal{V},\mathcal{E}(t))$, where 
$\mathcal{V}=\{1,\ldots,N\}$. The edge set $\mathcal{E}(t) \subseteq \mathcal{V} \times \mathcal{V}$ is induced by inter-robot distances as follows. Let 
$R_\mathrm{c} \in \mathbb{R}_{>0}$ denote the communication range with $R_\mathrm{c} > R_\mathrm{s}$. Given the joint robot state 
$\mathbf{x}$, we define the communication set as
\begin{align}
    h^\mathrm{c}_{i,j}(\mathbf{x}) 
        &= R_\mathrm{c}^2 - \bigl\| \mathbf{x}_i - \mathbf{x}_j \bigr\|^2,
        \quad \forall (v_i,v_j) \in \mathcal{E}(t), \\
    \mathcal{H}^\mathrm{c}_{i,j} 
        &= \bigl\{\mathbf{x} \in \mathbb{R}^{dN} : 
            h^\mathrm{c}_{i,j}(\mathbf{x}) \geq 0 \bigr\}.
            \label{eq:communication_constraint}
\end{align}
An undirected edge $(v_i,v_j)\in\mathcal{E}(t)$ exists when robots $i$ and $j$ are within the communication range, i.e., $\mathbf{x}(t)\in\mathcal{H}^\mathrm{c}_{i,j}$. The graph $\mathcal{G}(t)$ is \textit{globally connected} if every pair of vertices is connected by a path.
Note that, unlike conventional setting that require $\mathcal{G}(t)$ to be globally connected at all times, we assume edges in $\mathcal{E}(t)$ are short range and only need to be established temporarily (e.g., for exchanging high-bandwidth data periodically) so that global connectivity can be enforced over a prescribed time window through a sequence of edge reconnection. Coordination during disconnection is assumed to be maintained through external mechanisms that are not explicitly modeled in this work, e.g., via an auxiliary communication layer with a much larger effective range for continuous lightweight coordination information exchange (e.g., states and control signals) \citep{rouvcek2019darpa}.

\begin{definition}\textbf{Union of graph}.\label{def:union_graph}
Given a time-varying communication graph $\mathcal{G}(t)$, a time window length $T > 0$ and a starting time $t_0 \ge 0$, we define the \textbf{union of graph} $\mathcal{G}^\mathrm{u}(t_0)$ of $\mathcal{G}(t)$ over the interval $[t_0, t_0+T]$ as
\begin{equation}
\label{eq:union_graph}
   \mathcal{G}^\mathrm{u}(t_0, T) \triangleq
    \Bigl(\mathcal{V},\;\bigcup_{\tau \in [t_0,\,t_0+T]} \mathcal{E}(\tau)\Bigr) =(\mathcal{V},\mathcal{E}^{\mathrm{u}}).
\end{equation}
An edge $(v_i,v_j)$ exists in $\mathcal{G}^\mathrm{u}(t_0,T)$, if it is connected at least once during the time interval $[t_0, t_0+T]$.
\end{definition}
Based on the union graph, we define joint global connectivity as follows.

\begin{definition}\textbf{Joint global connectivity}.
\label{def:joint-twgc}
Given a start time $t_0$ and a time window length $T$,
the communication graph $\mathcal{G}$ is said to satisfy 
\textbf{joint global connectivity} over $[t_0,t_0+T]$, if the collection of all edges that appear during $[t_0, t_0+T]$ contains a path between any two vertices, i.e., union of graph $\mathcal{G}^\mathrm{u}(t_0,T)$ is connected. Such a communication graph $\mathcal{G}$ is called \textbf{jointly globally connected}.
\end{definition}

As a result, joint global connectivity ensures that information available at $t=t_0$ can be propagated from each robot to all others by $t=t_0+T$, even if $\mathcal{G}(t)$ is disconnected at intermediate times. This allows robots to temporarily split into separate components for efficient task execution while maintaining information flow.

Given a prescribed globally connected union graph  
$\bar{\mathcal{G}}^\mathrm{u} (t_0, T)\\=(\mathcal{V},\bar{\mathcal{E}}^\mathrm{u})$, 
we define the desired set for satisfying joint global connectivity as
\begin{align}\label{eq:Hset}
\mathcal{H}^\mathrm{c}_T\bigl(\bar{\mathcal{G}}^{\mathrm{u}}(t_0,T), t_0\bigr)
    =
    \Bigl\{
        \mathbf{x}\in\mathbb{R}^{dN} &| \forall (i,j)\in\bar{\mathcal{E}}^{\mathrm{u}},
        \\\exists \tau\in[t_0,t_0+T],\; \notag
        &
        \mathbf{x}(\tau)\in\mathcal{H}^\mathrm{c}_{i,j} 
    \Bigr\}.
\end{align}
\vspace{-0.8cm}
\subsection{Problem Statement} 
\label{sec:problem}
In this paper, we assume 
each robot $i$ has been assigned a nominal task-related controller $\mathbf{u}^\mathrm{nom}_i \in \mathbb{R}^q$, and denote their joint nominal controller by 
$\mathbf{u}^\mathrm{nom} = [\mathbf{u}_1^\mathrm{nom},\ldots,\mathbf{u}_N^\mathrm{nom}] \in \mathbb{R}^{qN}$.  
The initial state $\mathbf{x}_0$ lies in the safety set $\mathcal{H}^\mathrm{s}$. Given the real-time proximity-based communication graph $\mathcal{G}(t) = (\mathcal{V}, \mathcal{E}(t))$ and a predetermined globally connected union of graph $\bar{\mathcal{G}}^{\mathrm{u}}(t_0,T) = (\mathcal{V}, \bar{\mathcal{E}}^{\mathrm{u}})$ specifying the desired edge combination to be periodically realized for any $k$th 
time window $[t_0+kT,t_0+(k+1)T),\forall k\in \mathbb{Z}^{\ge 0}$, the \textbf{objective} is to design a joint multi-robot controller $\mathbf{u}(t)$ such that (i) the safety constraints are satisfied, (ii)
the union of communication graphs $\mathcal{G}(t)$ for $t\in [t_0+kT,t_0+(k+1)T)$ contains all edges in $\bar{\mathcal{E}}^{\mathrm{u}}$ at least once to satisfy joint global connectivity with respect to $\mathcal{G}(t)$ (i.e., $\bar{\mathcal{G}}^{\mathrm{u}}(t_0,T)$ is connected), and (iii) the resulting control deviation from the nominal joint controller $\mathbf{u}^\mathrm{nom}$ is minimized. Formally, at each control update, we consider the following Quadratic Programming (QP) problem:
\begin{align}
    \mathbf{u}^* 
    &= \text{argmin}_{\mathbf{u}} \; \sum_{i=1}^{N} \bigl\| \mathbf{u}_i - \mathbf{u}^\mathrm{nom}_i \bigr\|^2 
    \label{eq:twgc_obj} \\
    \text{s.t.}
    & \mathbf{u} \in \mathcal{B}^\mathrm{s}(\mathbf{x}), \mathbf{x} \in \mathcal{H}^\mathrm{c}_T\bigl(\bar{\mathcal{G}}^{\mathrm{u}}(t_0,T), t_0+kT\bigr), \forall k\in \mathbb{Z}^{\ge 0}
    \label{eq:twgc_conn_const} \\
    & \|\mathbf{u}_i\| \le \alpha_i,\; \forall i = 1,\ldots,N.
    \label{eq:twgc_input_bound} 
\end{align}
where $\alpha_i > 0$ denotes the control bound for robot $i$.

\section{Method}
\subsection{Adaptive Prescribed-Time CBF}\label{sec:adp_cbf}

To enforce the desired reconnection condition in~\eqref{eq:Hset}, we first derive control constraints for each selected edge. When an edge $(v_i,v_j)$ is activated with $h^\mathrm{c}_{i,j}(\mathbf{x})<0$, the controller must drive the robot pair back into the communication set before the prescribed deadline. Prescribed-time control barrier functions (PT-CBF)~\citep{abel2023prescribed} provide a natural tool for this purpose.
\begin{definition}\textbf{Blow-up function}.
A function $\eta:[a,b)\to\mathbb{R}_{\ge0}$ with $a<b$ is called a \textbf{blow-up function} if it is continuous, strictly increasing, and satisfies
$\lim_{t\to b^-}\eta(t)=+\infty$.
\end{definition}
Using this definition, we now recall the standard definition of a prescribed-time control barrier function (PT-CBF).

\begin{lemma}\textbf{Prescribed-time control barrier function (PT-CBF)}. [Summarized from \citep{abel2023prescribed}]\label{lem:prescribed_cbf}
   Given a team of robots with each robot's dynamics defined by~\eqref{eq:dynamics}, let $h:\mathbb{R}^n\to\mathbb{R}$ be continuously differentiable and define the safe set $\mathcal{C}=\{\mathbf{x}\mid h(\mathbf{x})\ge0\}$. The function $h(\mathbf{x})$ is a \textbf{PT-CBF} over the prescribed horizon $[t_0,t_0+T_p)$ if there exist $c>0$ and a blow-up function $\mu(t)$ such that  
\begin{equation}
    \dot{h}(\mathbf{x}) + c \,\mu(t)\, h(\mathbf{x}) \;\ge\; 0,
    \quad \forall t \in [t_0, t_0 + T_p).
    \label{eq:ptcbf}
\end{equation}
If \eqref{eq:ptcbf} holds, the system enters $\mathcal{C}$ no later than $t_0+T_p$, even when starting from $h(\mathbf{x}_0)<0$.
\end{lemma}

However, PT-CBF requires $T_p$ to be specified a priori, which is restrictive in state-dependent settings. To address this, we introduce the Finite-Time Control Barrier Function (FT-CBF) framework to obtain a computable convergence-time bound.

\begin{lemma}\textbf{Finite-time control barrier function (FT-CBF)}. 
[Summarized from \citep{li2018formally}]
\label{lem:ftcbf}
 Given a team of robots with each robot's dynamics defined by~\eqref{eq:dynamics}, let $\mathcal{C}=\{\mathbf{x}\in\mathbb{R}^n\mid h(\mathbf{x})\ge0\}$ be a desired set defined by a continuously differentiable function 
$h:\mathbb{R}^n\to\mathbb{R}$. 
The function $h$ is an \textbf{FT-CBF} on a set $\mathcal{D}$ if there exist $c>0$ and $\rho\in[0,1)$ such
\begin{equation}
\dot{h}(\mathbf{x},\mathbf{u})
+
c\cdot\operatorname{sign}(h(\mathbf{x}))
|h(\mathbf{x})|^{\rho}
\ge 0.
\label{eq:ftcbf}
\end{equation}
Any continuous controller satisfying~\eqref{eq:ftcbf} renders $\mathcal{C}$ forward invariant when $\mathbf{x}_0\in\mathcal{C}$, and drives $\mathbf{x}(t)$ into $\mathcal{C}$ within a finite time $T_\mathrm{FTCBF}=\frac{1}{c(1-\rho)}|h(\mathbf{x}_0)|^{1-\rho}$ when $\mathbf{x}_0\in\mathcal{D}\setminus\mathcal{C}$.
\end{lemma}

Unlike PT-CBF, FT-CBF guarantees convergence to the safe set within a finite time bound $T_{\mathrm{FTCBF}}$ that depends on the initial state rather than a user-specified constant. However, this bound is often conservative. 

Motivated by this, we propose an Adaptive Prescribed-Time Control Barrier Function (Adaptive PT-CBF), which determines the prescribed time $T_p$ online from the current state and thus reduces manual tuning. Theorem~\ref{theorem:adaptive_ptcbf} establishes the existence of a control input that ensures convergence within the resulting adaptive time $T_p$.

\begin{theorem}\textbf{Adaptive prescribed-time control barrier function}.
\label{theorem:adaptive_ptcbf}
Given a team of robots with each robot's dynamics defined by~\eqref{eq:dynamics}, let $h:\mathbb{R}^n \to \mathbb{R}$ be a continuously differentiable function defining the desired set $\mathcal{C} := \{\mathbf{x} \in \mathbb{R}^n \mid h(\mathbf{x}) \ge 0\}$. Then $h$ is an \textbf{adaptive prescribed-time control barrier function} if there exist real parameter $\rho \in (0,1)$ and $c > 0$, so that
\begin{equation}
    \dot{h}(\mathbf{x})
    + c\,\mu(t)\,\mathrm{sign}\bigl(h(\mathbf{x})\bigr)\,\bigl|h(\mathbf{x})\bigr|^{\rho}
    \;\ge\; 0,
    \label{eq:ptcbf_theorem}
\end{equation}
holds for all $t \in [0,T_p)$ where the time-varying scaling factor
$\mu(t)$ is defined as
\begin{equation}
   \mu(t)
   =
   \left(
       \frac{T_p}{T_p - t}
   \right)^m + \alpha,
   \quad t \in [0,T_p),
   \label{eq:mu_function}
\end{equation}
with $\alpha >0$ and $m \in 2\mathbb{N}$ (i.e., $m$ is a positive even integer). The prescribed time $T_p$ is chosen adaptively according to
\begin{equation}
    T_p =
    \frac{|h(\mathbf{x}_0)|^{\,1-\rho}}{c(1-\rho)\,\mu_0},
    \qquad
    \mu_0 = 1+\alpha > 1.
    \label{eq:Tp_adaptive}
\end{equation}
With this, the following properties hold for the corresponding closed-loop solution $\mathbf{x}(t)$: 1) (\textbf{Forward invariance}) If $\mathbf{x}_0 \in \mathcal{C}$, then
    $h(\mathbf{x}(t)) \ge 0$ for all $t \ge 0$; that is, the safe set $\mathcal{C}$ is forward invariant. 2) (\textbf{Prescribed-time reachability}) If $\mathbf{x}_0 \notin \mathcal{C}$, i.e.,
    $h(\mathbf{x}_0)<0$, then there exists a time $t^\star \le T_p$ such that
    $h(\mathbf{x}(t^\star)) = 0$, and $h(\mathbf{x}(t)) \ge 0$ for all $t \ge t^\star$. 
\end{theorem}

\begin{proof}
When $h(\mathbf{x})<0$, \eqref{eq:ptcbf_theorem} with
$\mathrm{sign}(h)=-1$ becomes $\dot{h}(\mathbf{x}) \;\ge\; c\,\mu(t)\,\bigl|h(\mathbf{x})\bigr|^{\rho}$. Define $V(\mathbf{x}) := |h(\mathbf{x})| = -h(\mathbf{x})$ in this region. Then
$\dot{V}(\mathbf{x}) = -\dot{h}(\mathbf{x}) \le -c\,\mu(t)\,V(\mathbf{x})^{\rho}$.
Since $\mu(t) \ge \mu_0$ for all $t \in [0,T_p(\mathbf{x}_0))$, we obtain the conservative inequality $\dot{V}(\mathbf{x}) \;\le\; -c\,\mu_0\,V(\mathbf{x})^{\rho}, \quad 0<\rho<1$.
By separation of variables and integration from $V(\mathbf{x}_0)=|h(\mathbf{x}_0)|$ down to $0$,
one obtains convergence bound
    $T_{\mathrm{finite}}
    \;\le\;
    \frac{|h(\mathbf{x}_0)|^{\,1-\rho}}{c(1-\rho)\,\mu_0},$
which by construction, equals $T_p$ in \eqref{eq:Tp_adaptive}. This proves 2):
starting outside $\mathcal{C}$, the trajectory reaches $h(\mathbf{x})=0$ within the prescribed time
$T_p(\mathbf{x}_0)$ and remains in $\mathcal{C}$ thereafter. For 1), when $h(\mathbf{x}) \ge 0$ we have $\mathrm{sign}(h)=1$ and
\eqref{eq:ptcbf_theorem} yields $\dot{h}(\mathbf{x}) \;\ge\; -c\,\mu(t)\,h(\mathbf{x})^{\rho}$,
which is a standard control barrier condition \citep{ames2019control}. Thus, if $h(\mathbf{x}_0)\ge 0$, trajectory
remains in $\mathcal{C}$ for all $t \ge 0$, establishing set forward invariance property.
\end{proof}

\begin{remark}
\label{rem:Tp_vs_FTCBF}
Under Lemma~\ref{lem:ftcbf}, FT-CBF guarantees convergence within $T_{\mathrm{FTCBF}}$. Since~\eqref{eq:Tp_adaptive} gives $T_p=T_{\mathrm{FTCBF}}/\mu_0$ with $\mu_0=1+\alpha>1$, we have $T_p<T_{\mathrm{FTCBF}}$. Hence, adaptive PT-CBF yields a less conservative convergence-time bound than FT-CBF.
\end{remark}

With Theorem~\ref{theorem:adaptive_ptcbf}, given any desired union of graph
$\bar{\mathcal{G}}^{\mathrm{u}}(t_0,T)$, the corresponding control set is defined as
\begin{align}    \label{eq:ptcbf_condition}
    \mathcal{B}^{\mathrm{r}}&(\mathbf{x})
    = \Big\{ \mathbf{u}\in\mathbb{R}^{qN} \,\Big|\,
    \dot{h}_{i,j}^\mathrm{c}(\mathbf{x})
    +\\& c\,\mu(t)\,\operatorname{sign}\!\big(h_{i,j}^\mathrm{c}(\mathbf{x})\big)\,
      \lvert h_{i,j}^\mathrm{c}(\mathbf{x}) \rvert^{\rho}
    \ge 0,\;
    \forall (v_i,v_j) \in \bar{\mathcal{E}}^{\mathrm{u}}(t)
    \Big\}. \notag
\end{align}
Any control input $\mathbf{u}(t) \in \mathcal{B}^{\mathrm{r}}(\mathbf{x}(t))$ guarantees that all desired pairwise reconnection edges satisfy the communication constraint within the prescribed time horizon.

\subsection{Spatio-temporal Reconnection Mechanism}
\label{sec:spatio_graph}

We introduce a scheduling mechanism that activates each edge connection constraint only when reconnection is needed and deactivates it after the edge is realized for flexible task execution. To decide \emph{when} each desired edge should reconnect, we first define a spatio-temporal edge weight for every $(v_i,v_j)\in\bar{\mathcal{E}}^\mathrm{u}$:
{\small
\begin{equation}
\begin{aligned}
w_{i,j}(t)
= {}& \dot{h}^\mathrm{c}_{i,j}\bigl(\mathbf{x}(t),
\mathbf{u}^\mathrm{nom}_i(t),
\mathbf{u}^\mathrm{nom}_j(t)\bigr)
- c\,\mu(0)\bigl|h^\mathrm{c}_{i,j}(\mathbf{x}(t))\bigr|^\rho
+ \phi(t),
\end{aligned}
\label{eq:weight}
\end{equation}}

where first two terms measure \textit{spatial} effort required for edge $(v_i,v_j)$ to reconnect under nominal controls, while $\phi(t)$ is blow-up \textit{temporal} term satisfying 
$\lim_{t\to (t_{\mathrm{max\_trigger}}^k)^-}\\\phi(t)=+\infty$. 
Thus, the reconnection urgency increases as the latest triggering time $t_{\mathrm{max\_trigger}}^k$ approaches.

The edge $(v_i,v_j)$ is activated when $w_{i,j}(t)>0$. Let $\tau_{i,j}^k$ denote the triggering time and $t_{i,j}^k$ denote the reconnection time satisfying $h^\mathrm{c}_{i,j}(\mathbf{x}(t_{i,j}^k))\ge0$. By Theorem~\ref{theorem:adaptive_ptcbf} and the definition of $t_{\mathrm{max\_trigger}}^k$, any edge triggered before $t_{\mathrm{max\_trigger}}^k$ reconnects within the same window, i.e., $t_{i,j}^k\le t_0+(k+1)T$.

Under the adaptive PT-CBF reconnection controller, once an edge $(v_i,v_j)$ is triggered at $\tau_{i,j}^k$, Theorem~\ref{theorem:adaptive_ptcbf} guarantees the existence of a first reconnection time
$t_{i,j}^k \in [\tau_{i,j}^k,\,\tau_{i,j}^k+T_p^{\max}]$
such that
$h^\mathrm{c}_{i,j}(\mathbf{x}(t_{i,j}^k))\ge0$.
Since each edge is triggered before the latest allowable triggering time
$t_{\mathrm{max\_trigger}}^k$, we have $t_{i,j}^k
\le
t_{\mathrm{max\_trigger}}^k+T_p^{\max}
=
t_0+(k+1)T$, which ensures that the reconnection is completed within the same time window.
After the edge is connected for the first time in the $k$th window, its corresponding constraint is deactivated for the remainder of window.

For any time $t$ in the $k_{\textrm{th}}$ time window $[t_0+kT,t_0+(k+1)T)$, we define the \emph{active edge set} as
\begin{equation}
\mathcal{E}_k^{\mathrm{act}}(t)
= \Bigl\{ (v_i,v_j) \in \bar{\mathcal{E}}^{\mathrm{u}}
\;\Big|\;
t \in [\tau_{i,j}^k,\, t_{i,j}^k) \Bigr\}.
\label{eq:active_edge_set}
\end{equation}
Thus, each edge constraint is active only between its triggering time $\tau_{i,j}^k$ and its first reconnection time $t_{i,j}^k$, and is deactivated for the rest of window after reconnection. The corresponding admissible set applies adaptive PT-CBF constraints only to active edges:
\begin{align}
&\mathcal{B}^{\mathrm{c},k}(\mathbf{x},t)
= \Big\{ \mathbf{u}\in\mathbb{R}^{qN} \,\Big|\,
    \dot{h}_{i,j}^\mathrm{c}\big(\mathbf{x}(t)\big)+ \\
    & c\,\mu(t)\,\operatorname{sign}\!\big(h_{i,j}^\mathrm{c}(\mathbf{x}(t))\big)\,
      \big\lvert h_{i,j}^\mathrm{c}(\mathbf{x}(t)) \big\rvert^{\rho}
    \ge 0,\;
   \forall (v_i,v_j) \in \mathcal{E}_k^{\mathrm{act}}(t)
    \Big\}. \notag
\end{align}
At any time $t\in [t_0+kT,t_0+(k+1)T)$, the control input is obtained by solving the following QP:
\begin{align}
\mathbf{u}^*(t)
&= \arg\min_{\mathbf{u}} \;\;
    \sum_{i=1}^{N} \bigl\| \mathbf{u}_i - \mathbf{u}_i^{\mathrm{nom}}(t) \bigr\|^2
    \label{eq:twgc_qp_obj} \\[1mm]
\text{s.t.}\quad
& \mathbf{u} \in \mathcal{B}^{\mathrm{s}}(\mathbf{x}), \;\mathbf{u} \in \mathcal{B}^{\mathrm{c},k}(\mathbf{x},t),
    \label{eq:twgc_qp_connect_set} \\[1mm]
& \|\mathbf{u}_i\| \le \alpha_i,\quad i = 1,\ldots,N.
    \label{eq:twgc_qp_bounds_set}
\end{align}
\noindent
Here $\mathcal{B}^{\mathrm{s}}(\mathbf{x})$ enforces collision avoidance, while $\mathcal{B}^{\mathrm{c},k}(\mathbf{x},t)$ enforces reconnection only for active edges. The resulting problem is convex QP with \emph{linear} control constraints, which can be solved efficiently using standard solvers.
\begin{proposition}\textbf{Joint global connectivity under the QP controller}.
\label{prop:joint_connectivity_qp}
Let $\bar{\mathcal{G}}^{\mathrm{u}}(t_0+kT,T)=(\mathcal{V},\bar{\mathcal{E}}^{\mathrm{u}})$ be a desired globally connected union graph. Under QP controller \eqref{eq:twgc_qp_obj}--\eqref{eq:twgc_qp_bounds_set}, with triggering rule \eqref{eq:weight} and active edge set \eqref{eq:active_edge_set}, every edge $(v_i,v_j)\in\bar{\mathcal{E}}^{\mathrm{u}}$ is realized at least once in each window $[t_0+kT,t_0+(k+1)T)$. Consequently, $\bar{\mathcal{E}}^{\mathrm{u}}
\subseteq
\tilde{\mathcal{E}}^{\mathrm{u}}_k
=
\bigcup_{t\in[t_0+kT,t_0+(k+1)T)}
\{(v_i,v_j)\mid h^\mathrm{c}_{i,j}(\mathbf{x}(t))\ge0\}$, and the resulting union graph
$\tilde{\mathcal{G}}^{\mathrm{u}}_k=(\mathcal{V},\tilde{\mathcal{E}}^{\mathrm{u}}_k)$
is globally connected for every $k$. Hence, induced communication graph is \textbf{jointly globally connected} over every reconnection window.
\end{proposition}
\begin{proof}
Consider a reconnection window interval $[t_0+kT,t_0+(k+1)T)$ and any desired edge $(v_i,v_j)\in\bar{\mathcal{E}}^{\mathrm{u}}$. If $w_{i,j}(t_0+kT)>0$, the edge is triggered immediately; otherwise, the blow-up temporal term $\phi(t)$ in \eqref{eq:weight} ensures that there exists a triggering time
$\tau_{i,j}^k\in[t_0+kT,t_{\mathrm{max\_trigger}}^k)$
such that $w_{i,j}(\tau_{i,j}^k)>0$. By \eqref{eq:active_edge_set}, the edge remains active until its first reconnection time, so its adaptive PT-CBF constraint is enforced by the QP through $\mathcal{B}^{\mathrm{c},k}(\mathbf{x},t)$. By Theorem~\ref{theorem:adaptive_ptcbf}, there exists
$t_{i,j}^k\in[\tau_{i,j}^k,\tau_{i,j}^k+T_p^{\max}]$
such that
$h^\mathrm{c}_{i,j}(\mathbf{x}(t_{i,j}^k))\ge0$.
Since $\tau_{i,j}^k<t_{\mathrm{max\_trigger}}^k$ and
$t_{\mathrm{max\_trigger}}^k+T_p^{\max}=t_0+(k+1)T$, this reconnection occurs before the window ends. Therefore, every edge in $\bar{\mathcal{E}}^{\mathrm{u}}$ belongs to $\tilde{\mathcal{E}}^{\mathrm{u}}_k$. Since $\bar{\mathcal{G}}^{\mathrm{u}}(t_0+kT,T)$ is globally connected and
$\bar{\mathcal{E}}^{\mathrm{u}}\subseteq\tilde{\mathcal{E}}^{\mathrm{u}}_k$, the union graph
$\tilde{\mathcal{G}}^{\mathrm{u}}_k(t_0+kT,T)$ is globally connected. This proves joint global connectivity over every window.
\end{proof}
\vspace{-0.15cm}
\begin{figure*}[!htbp]
    \centering
    \begin{subfigure}{0.28\textwidth}
\includegraphics[width=\textwidth]{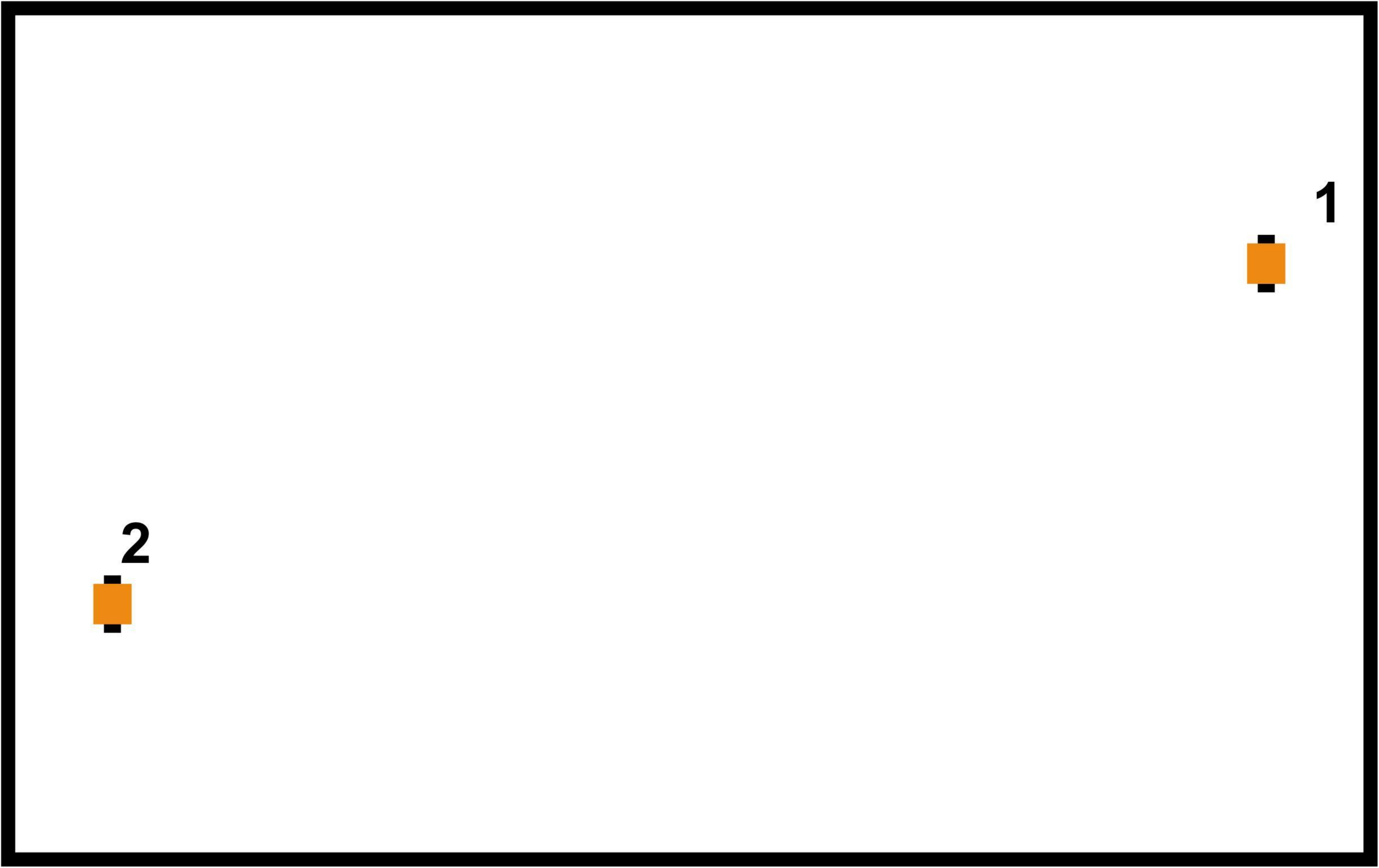}
    \caption{\textbf{Adaptive PT-CBF}, $t=0\textrm{s}$}
    \label{fig:adaptive_pt_cbf_1}
  \end{subfigure}
  \begin{subfigure}{0.28\textwidth}
\includegraphics[width=\textwidth]{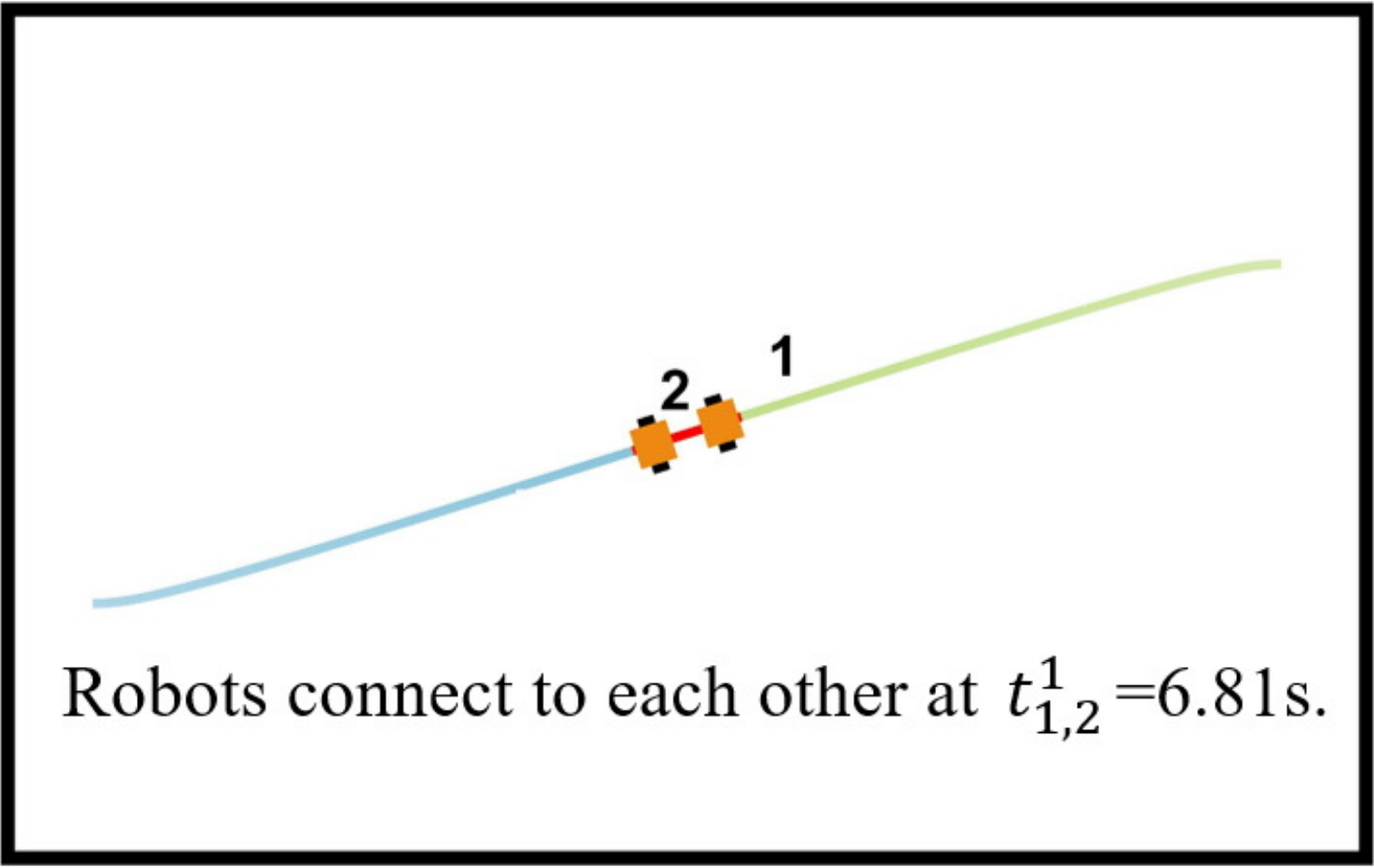}
    \caption{\textbf{Adaptive PT-CBF}, $t=6.81\textrm{s}$}
    \label{fig:adaptive_pt_cbf_2}
  \end{subfigure}
    \begin{subfigure}{0.28\textwidth}
\includegraphics[width=\textwidth]{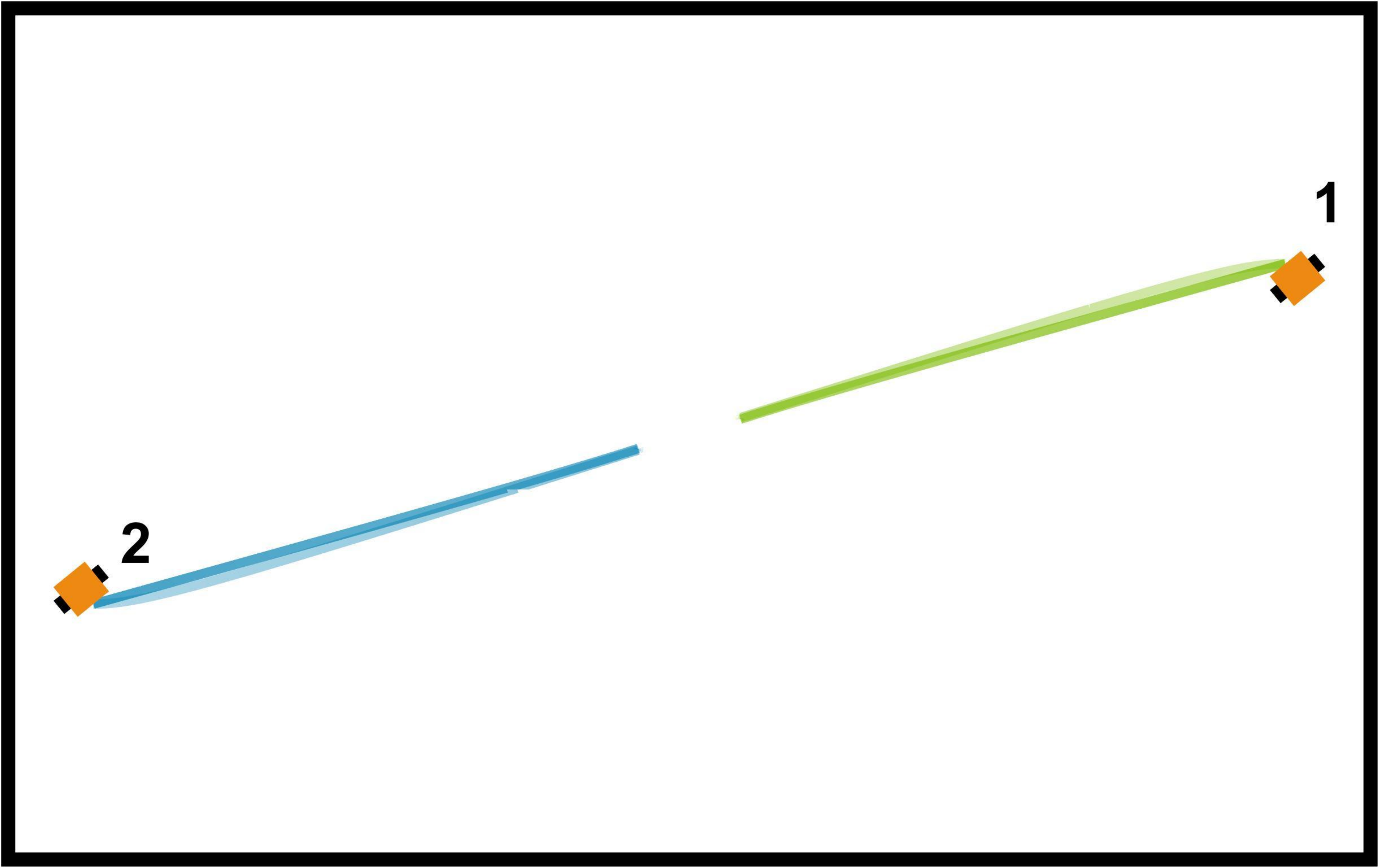}
    \caption{\textbf{Adaptive PT-CBF}, $t=15\textrm{s}$}
    \label{fig:adaptive_pt_cbf_3}
  \end{subfigure}
     \begin{subfigure}{0.28\textwidth}
\includegraphics[width=\textwidth]{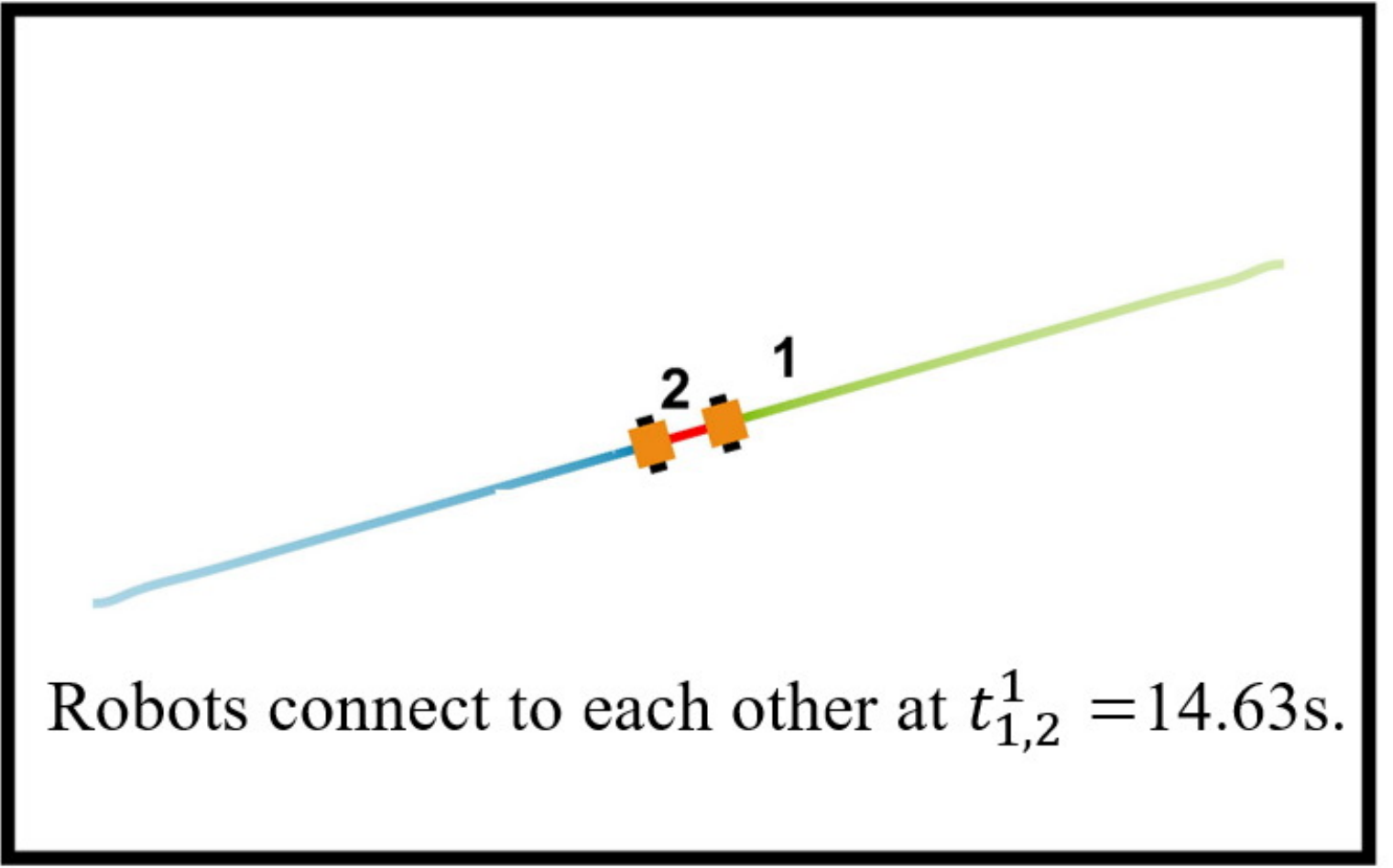} 
    \caption{FT-CBF, $t=15\textrm{s}$}
    \label{fig:ft_cbf}
  \end{subfigure}
       \begin{subfigure}{0.277\textwidth}
\includegraphics[width=\textwidth]{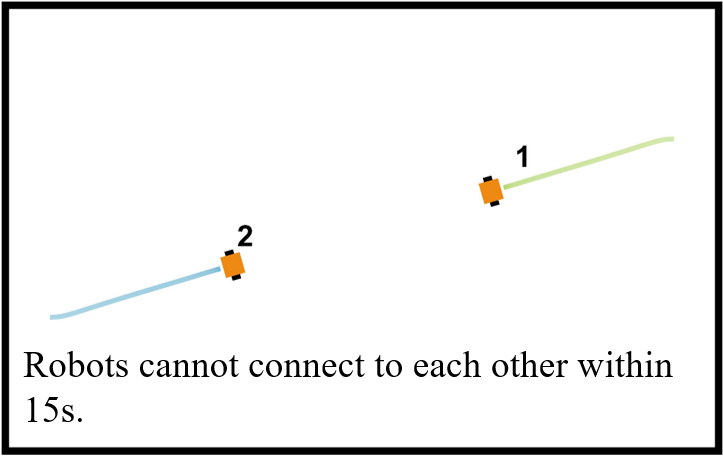}
        \caption{PT-CBF ($T_p=0.5\textrm{s}$), $t=15\textrm{s}$}
    \label{fig:pt_cbf_t_small}
  \end{subfigure}
   \begin{subfigure}{0.28\textwidth}
\includegraphics[width=\textwidth]{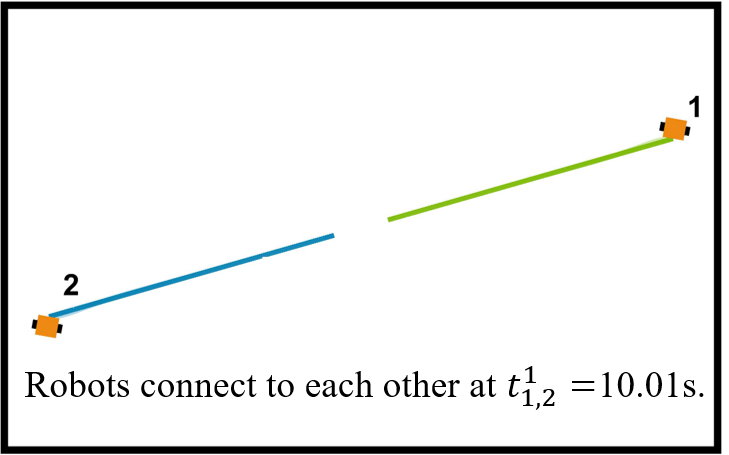}
\caption{PT-CBF ($T_p=15\textrm{s}$), $t=15\textrm{s}$}
    \label{fig:pt_cbf_t_large}
  \end{subfigure}
\caption{Snapshots comparing reconnection performance for \textbf{Adaptive PT-CBF (ours)}, FT-CBF, and PT-CBF.}
  \label{fig:simulation}
\end{figure*}
\begin{figure*}[!htbp]
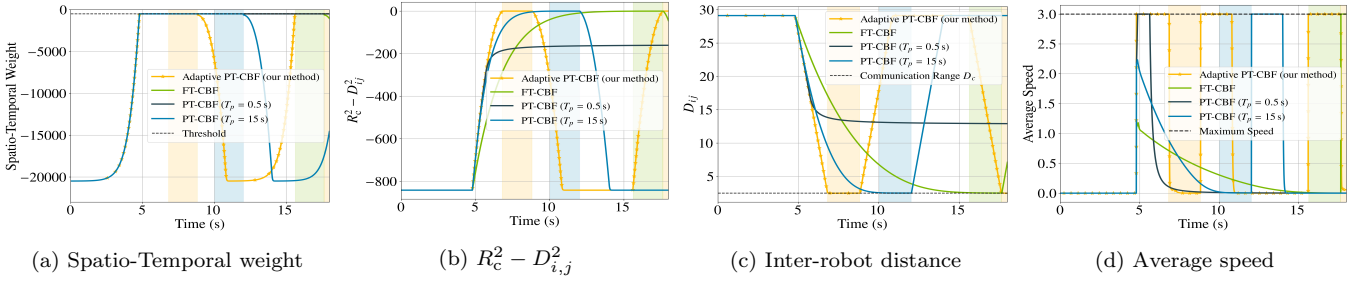

\centering
   \begin{subfigure}{0.24\textwidth}
\includegraphics[width=\textwidth,height=0.7\textwidth]{image/compare_edge_weight_python.pdf}
    \caption{Spatio-Temporal weight}
    \label{fig:edge_weight}
  \end{subfigure}
   \begin{subfigure}{0.24\textwidth}
\includegraphics[width=\textwidth,height=0.7\textwidth]{image/compare_hij_python.pdf}
    \caption{$R_\mathrm{c}^2-D_{i,j}^2$}
    \label{fig:h_function}
  \end{subfigure}
  \begin{subfigure}{0.24\textwidth}
\includegraphics[width=\textwidth,height=0.7\textwidth]{image/compare_interrobot_distance_python.pdf}
    \caption{Inter-robot distance}
    \label{fig:distance}
  \end{subfigure}
   \begin{subfigure}{0.24\textwidth}
\includegraphics[width=\textwidth,height=0.7\textwidth]{image/compare_velocity_python.pdf}
    \caption{Average speed}
    \label{fig:control_deviation}
  \end{subfigure}
\caption{Comparison of the reconnection process under Adaptive PT-CBF (ours), FT-CBF, and PT-CBF.  
(a) The spatio-temporal weight.  
(b) Communication-barrier value $R_\mathrm{c}^2-D_{i,j}^2$.  
(c) Inter-robot distance relative to the communication range $R_\mathrm{c}$.  
(d) Average robot speed.  }
\label{fig:matlab_numerical}
\end{figure*}
\begin{figure*}[!htbp]
    \centering
    \begin{subfigure}{0.3\textwidth}
\includegraphics[width=\textwidth,height=0.5\textwidth]{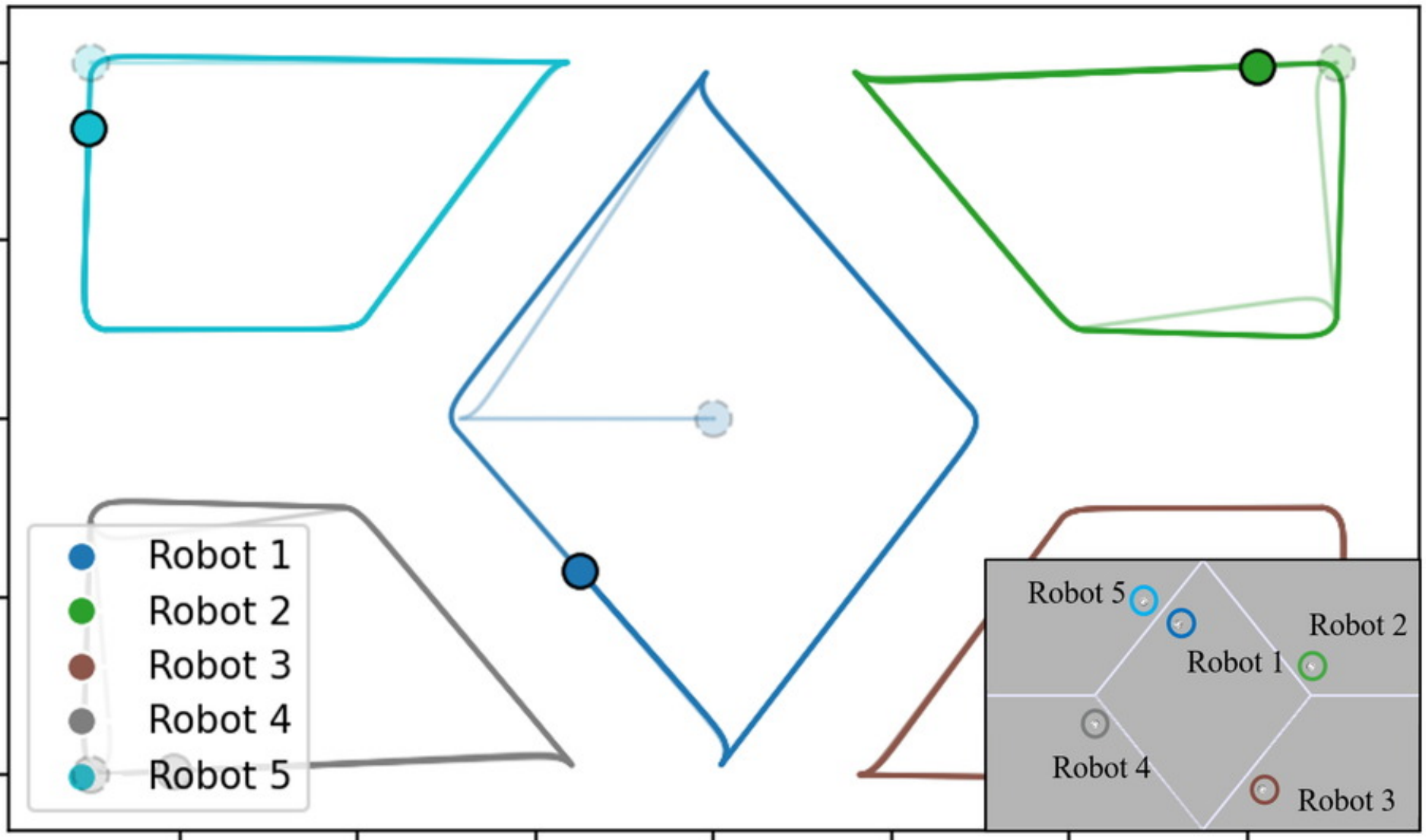}
    \caption{Original task}
    \label{fig:cop_1}
  \end{subfigure}
  \begin{subfigure}{0.3\textwidth}
\includegraphics[width=\textwidth,height=0.5\textwidth]{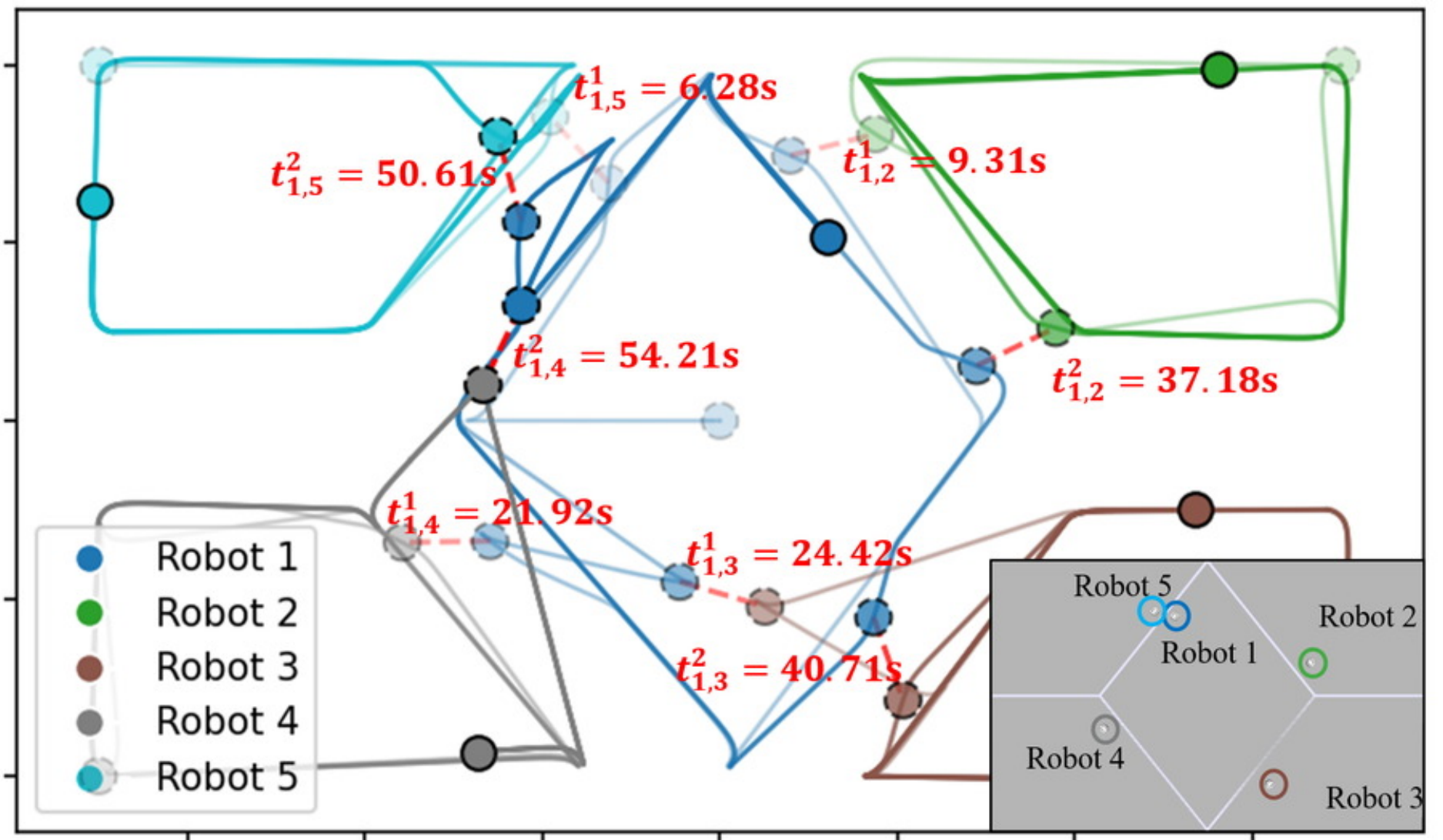}
    \caption{\textbf{Our Adaptive PT-CBF}}
\label{fig:cop_2}
  \end{subfigure}
    \begin{subfigure}{0.3\textwidth}
\includegraphics[width=\textwidth,height=0.5\textwidth]{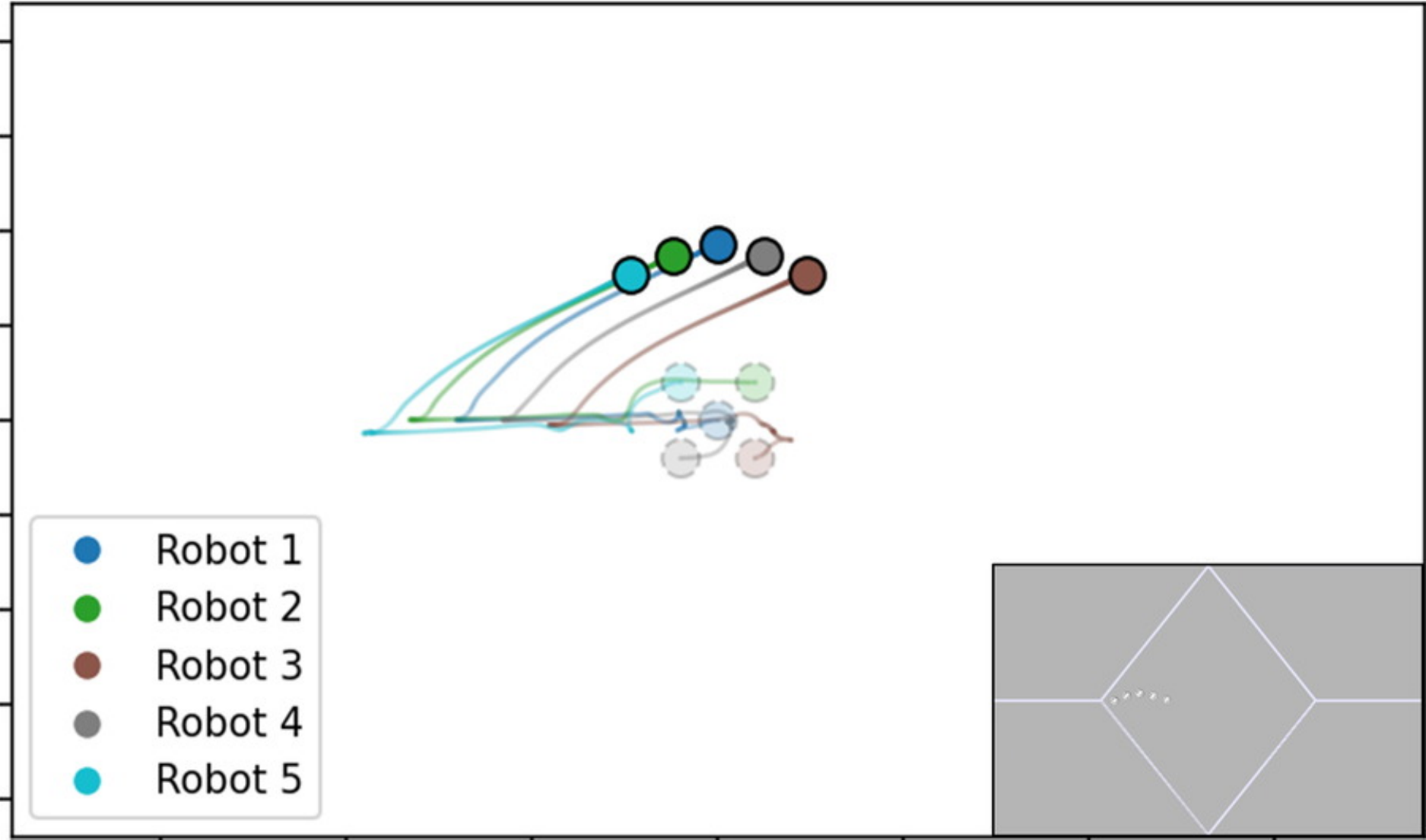}
    \caption{MCCST}
    \label{fig:cop_3}
  \end{subfigure}
\caption{Simulation snapshots and trajectories: bottom-right snapshot captured at $t=6.2\text{s}$, and all trajectories for the (a) Original Task, (b) \textbf{Adaptive PT-CBF (ours)}, and (c) MCCST baseline.}
  \label{fig:simulation_cop}
\end{figure*}
\begin{figure}[!htbp]
\centering
   \begin{subfigure}{0.2\textwidth}
\includegraphics[width=\textwidth]{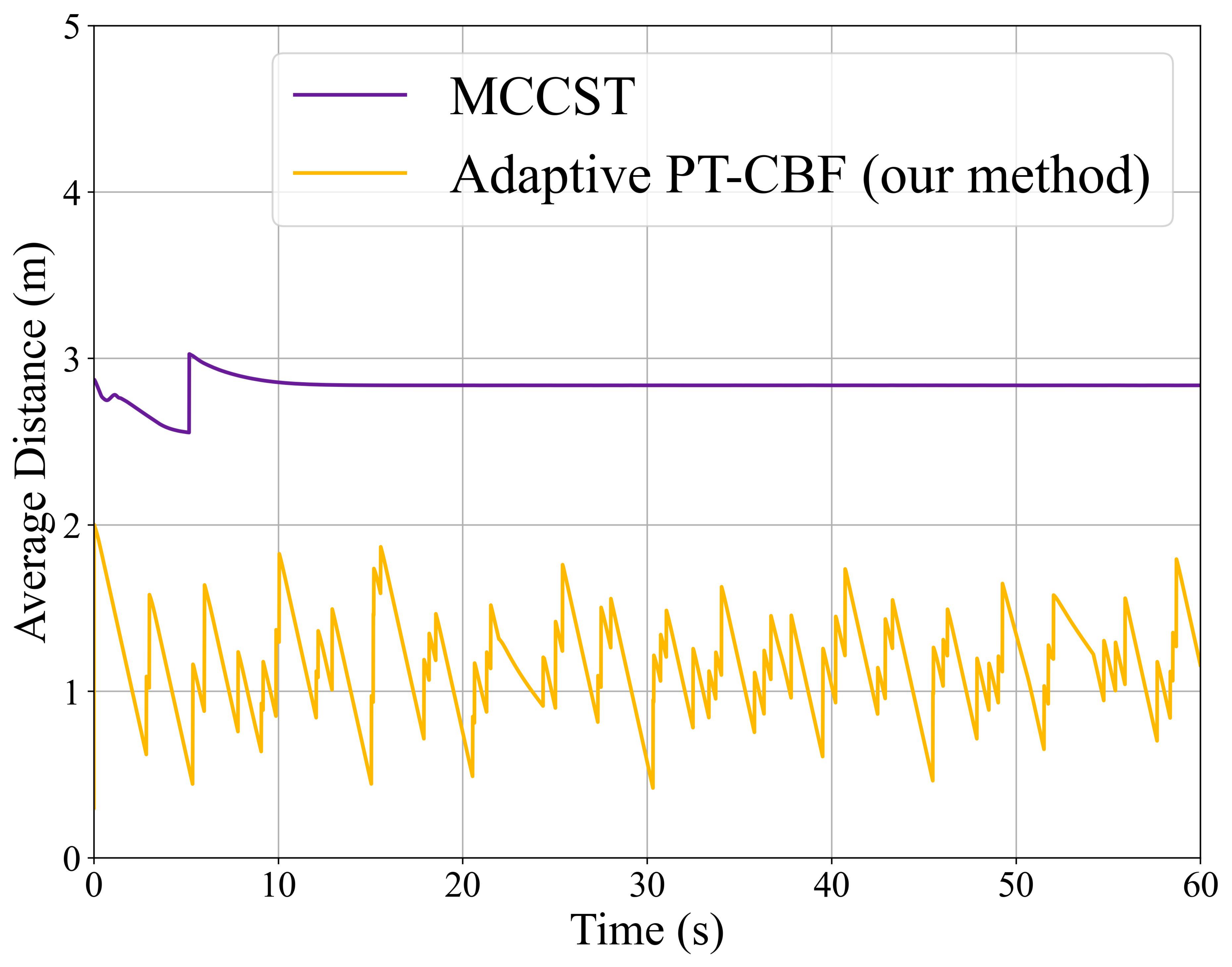}
    \caption{}
    \label{fig:distance_to_target}
  \end{subfigure}
   \begin{subfigure}{0.2\textwidth}
\includegraphics[width=\textwidth]{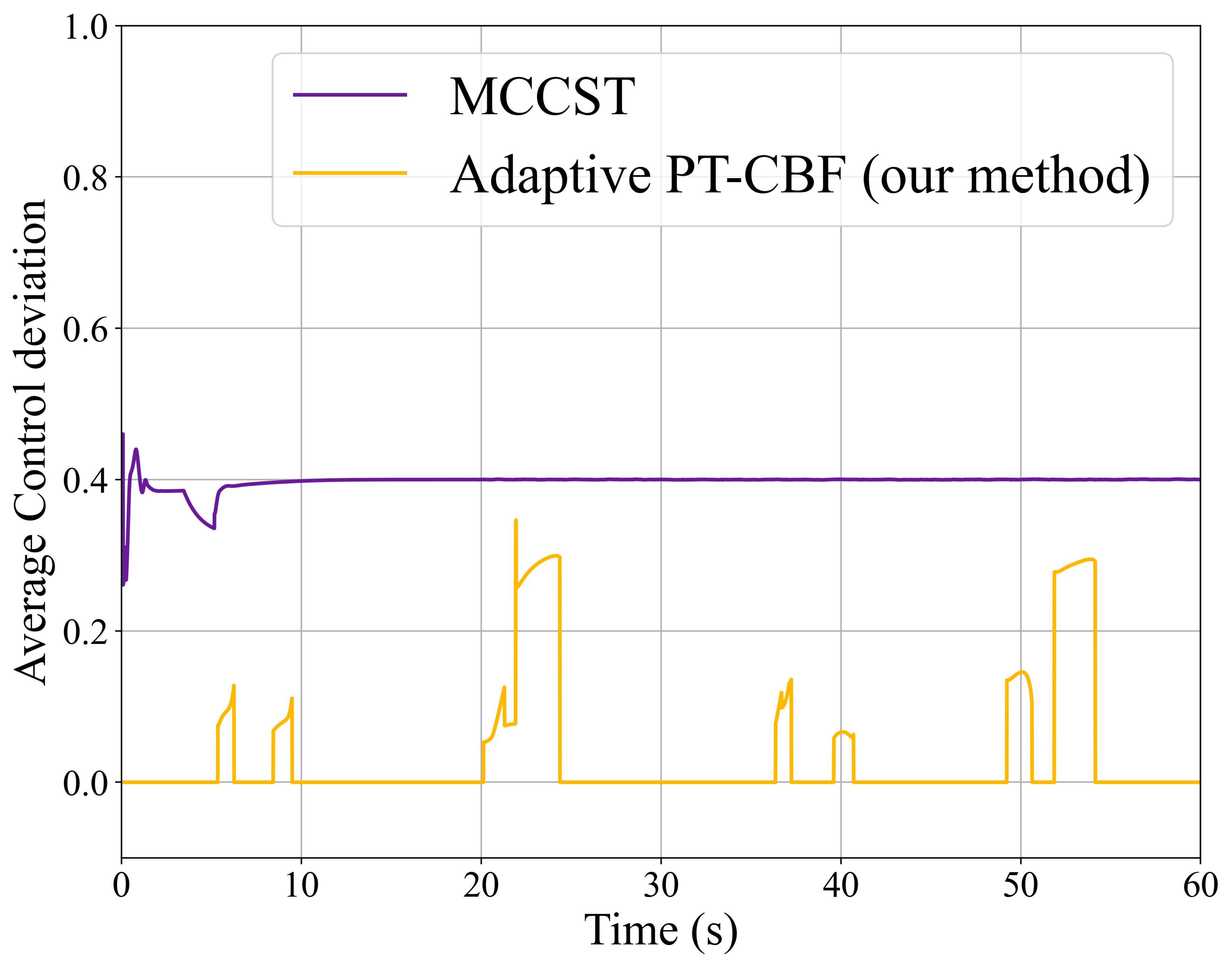}
    \caption{}
    \label{fig:cop_control_deviation}
  \end{subfigure}
\caption{Comparison between our proposed adaptive PT-CBF and the MCCST baseline. (a) Average distance to the assigned patrol target. (b) Average control deviation from the nominal controller $\frac{1}{5}
      \sum_{i=1}^{5} 
      \left\| 
          \mathbf{u}_{i}^{\text{nom}} - 
          \mathbf{u}_{i}^{\text{*}}
      \right\|$.}
\label{fig:cop_numerical}
\end{figure}
\begin{figure}[!htbp]
    \centering
\includegraphics[width=0.35\textwidth]{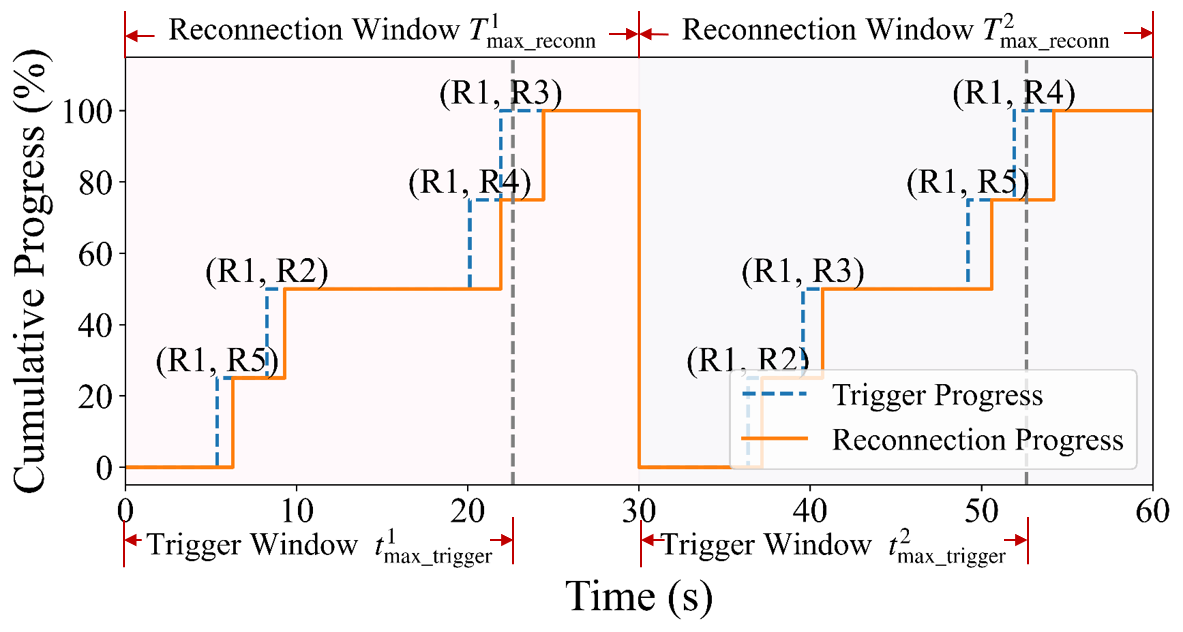}
    \caption{Cumulative trigger and reconnection progress of required communication edges.}
    \label{fig:edge_completion}
\end{figure}
\begin{figure}[!htbp]
\centering
   \begin{subfigure}{0.17\textwidth}
\includegraphics[width=\textwidth]{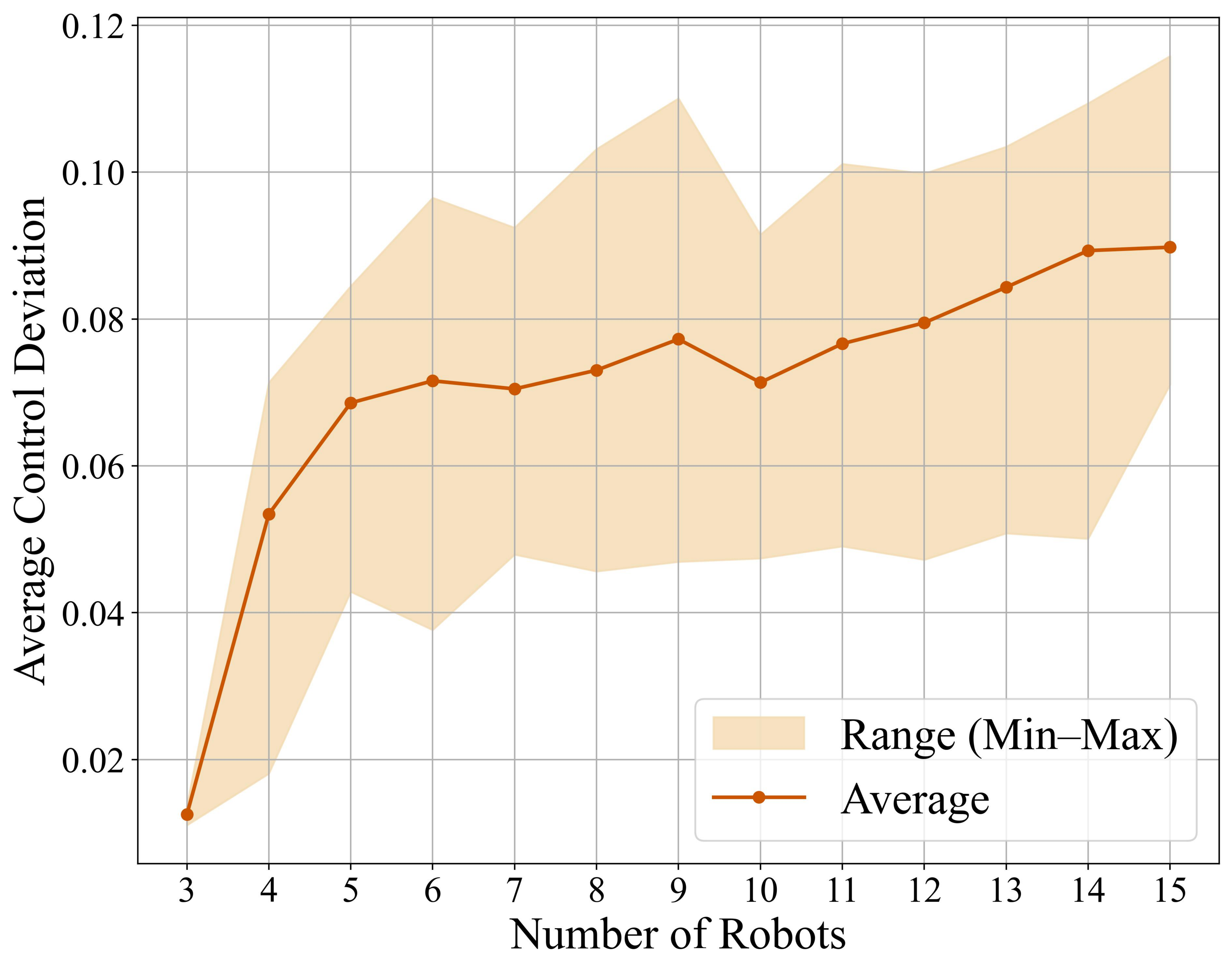}
    \caption{}
    \label{fig:quantitative_result_number_robots}
  \end{subfigure}
   \begin{subfigure}{0.17\textwidth}
\includegraphics[width=\textwidth]{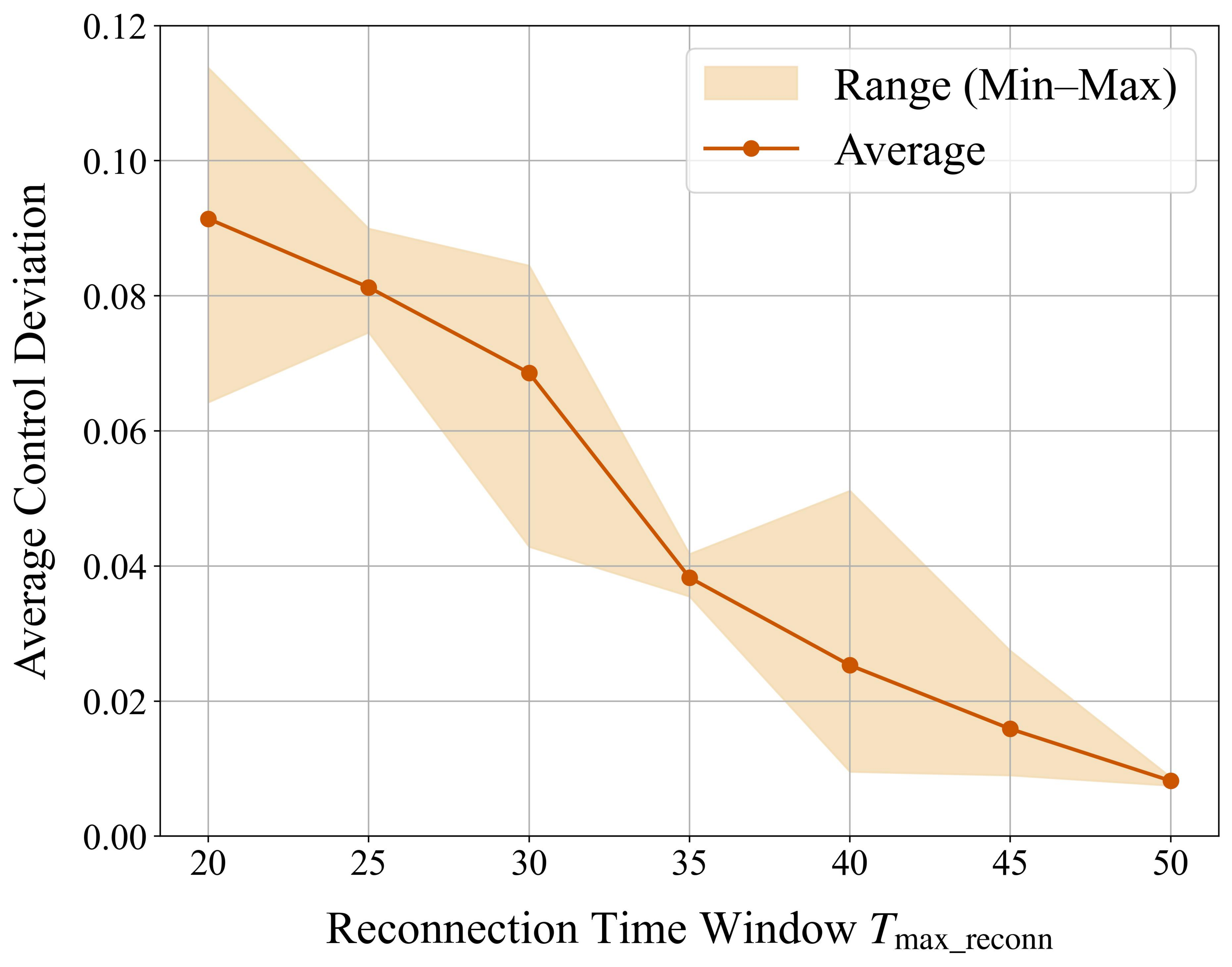}
    \caption{}
    \label{fig:quantitative_result_reconnection_window}
  \end{subfigure}
\caption{Quantitative results of our Adaptive PT-CBF. (a) average control deviation for varying team sizes. (b) average control deviation for different reconnection-window durations $T_{\mathrm{max\_reconn}}$.}
\label{fig:cop_quan}
\end{figure}

\vspace{-0.15cm}
\section{Results}
\vspace{-0.15cm}
\subsection{Simulation Example}
\vspace{-0.15cm}

\noindent \textbf{\textit{Case Study 1:}} To evaluate the proposed method, we first consider a two-robot reconnection scenario on the Robotarium platform~\citep{pickem2017robotarium}. 
The spatial component of the edge weight is kept constant, so the reconnection trigger depends only on the temporal term $\phi(t)$ in~\eqref{eq:weight}. 
This setting allows a controlled comparison among the proposed Adaptive PT-CBF, standard PT-CBF~\citep{abel2023prescribed} with $T_p=0.5\,\mathrm{s}$ and $T_p=15\,\mathrm{s}$, and FT-CBF~\citep{li2018formally}. 
The parameters in \eqref{eq:ptcbf}, \eqref{eq:ftcbf}, and \eqref{eq:ptcbf_condition} are set to $c=0.5$ and $\rho=0.9$, with communication range $R_\mathrm{c}=2.5\,\mathrm{m}$. 
Each robot follows single-integrator dynamics with bounded control input $\|\mathbf{u}_i\|\le 3\,\mathrm{m/s}$ and is nominally controlled to remain near its initial position. 
After each successful reconnection, the robots remain within communication range for $2\,\mathrm{s}$ to emulate information exchange, after which $\phi(t)$ is reset for the next cycle. We instantiate $\phi(t)$ as a logistic blow-up function to ensure triggering before $t_{\mathrm{max\_trigger}}=5\,\mathrm{s}$.

Fig.~\ref{fig:matlab_numerical} compares the three methods using four metrics. 
Since all methods use the same temporal component, Fig.~\ref{fig:matlab_numerical}(a) shows identical edge-weight evolution. 
However, Fig.~\ref{fig:matlab_numerical}(b)--(c) show that Adaptive PT-CBF drives the communication barrier and inter-robot distance to the threshold $R_\mathrm{c}$ earlier than the baselines. 
In particular, PT-CBF with $T_p=0.5\,\mathrm{s}$ fails to reconnect within the prescribed time, while $T_p=15\,\mathrm{s}$ leads to delayed response, illustrating the difficulty of manually selecting $T_p$. 
Fig.~\ref{fig:matlab_numerical}(d) further shows that Adaptive PT-CBF generates more effective motion, whereas FT-CBF weakens near the boundary and delays reconnection.

\noindent \textbf{\textit{Case Study 2:}}
We next evaluate the proposed method in a five-robot patrolling scenario. 
To ensure periodic information exchange, we predefine a target union graph 
$\bar{\mathcal{G}}^\mathrm{u}=(\mathcal{V},\bar{\mathcal{E}}^\mathrm{u})$, where Robot 1 must reconnect with the other four robots:
$\bar{\mathcal{E}}^\mathrm{u}=\{(v_1,v_2),(v_1,v_3),(v_1,v_4),(v_1,v_5)\}$.
The communication range is $R_\mathrm{c}=2.5\,\mathrm{m}$, and two consecutive $30\,\mathrm{s}$ reconnection windows are considered, within which each required edge must be established once. 
The optimized single-integrator controls from \eqref{eq:twgc_qp_obj} are mapped to unicycle commands.

For adaptive PT-CBF, we set $c=0.5$ and $\rho=0.9$. 
The temporal component in \eqref{eq:weight} is chosen as
$\phi(t)=\frac{0.4t}{t_{\mathrm{max\_trigger}}-t}+1$, 
which increases as the trigger deadline approaches. 
The maximum expected reconnection time is $7.36\,\mathrm{s}$, yielding 
$t_{\mathrm{max\_trigger}}^1=22.63\,\mathrm{s}$ and 
$t_{\mathrm{max\_trigger}}^2=52.63\,\mathrm{s}$ for the two windows. During the intervals $(0,\,22.63)$ and $(30,\,52.63)$, temporal gain $\phi(t)$ increases rapidly as $t$ approaches right endpoint of each window, ensuring that all required edges are activated before the trigger deadline.

We compare adaptive PT-CBF with the Minimum Connectivity Constraint Spanning Tree (MCCST) baseline~\citep{luo2020behavior}, which maintains persistent global connectivity. 
The reported metrics include average distance to target, control deviation, and completion progress of the required communication edges, shown in Figs.~\ref{fig:cop_numerical} and~\ref{fig:edge_completion}. 
As shown in Fig.~\ref{fig:simulation_cop}, adaptive PT-CBF allows robots to temporarily deviate from their nominal trajectories to reconnect while still completing their patrolling tasks. 
In contrast, MCCST preserves global connectivity at all times but causes excessive deviation from the nominal controller, leading to poor task progress and deadlock-like behavior.

\subsection{Quantitative Results}
We quantitatively evaluate the proposed adaptive PT-CBF framework in CoppeliaSim, with results summarized in Fig.~\ref{fig:cop_quan}. We vary the team size and the reconnection window length $T_{\mathrm{max\_reconn}}$, and test three initial configurations for each case. Fig.~\ref{fig:cop_quan}(a) shows that the average control deviation increases only slightly with team size, indicating that the added reconnection demands remain manageable. As shown in Fig.~\ref{fig:cop_quan}(b), longer windows reduce control deviation because robots have more time to reconnect with smaller control modifications. These results show that adaptive PT-CBF maintains limited control deviation across different team sizes and window lengths.
\vspace{-0.1cm}
\section{Conclusion and Future Work}
\vspace{-0.1cm}
We presented an adaptive prescribed-time CBF framework for spatio-temporal reconnection in multi-robot systems. The proposed method combines state-dependent prescribed-time selection with a spatio-temporal triggering mechanism, allowing robots to temporarily relax connectivity constraints while guaranteeing reconnection within each time window. Theoretical analysis and simulations show that the method achieves timely reconnection with lower control deviation than persistent-connectivity and baseline approaches. 
Future work will investigate dynamic and automatic selection of the desired union graph and reconnection schedule.
\bibliography{ifacconf}    

\begin{thebibliography}{11}
\providecommand{\natexlab}[1]{#1}
\providecommand{\url}[1]{\texttt{#1}}
\providecommand{\urlprefix}{URL }
\expandafter\ifx\csname urlstyle\endcsname\relax
  \providecommand{\doi}[1]{doi:\discretionary{}{}{}#1}\else
  \providecommand{\doi}{doi:\discretionary{}{}{}\begingroup \urlstyle{rm}\Url}\fi

\bibitem[{Abel et~al.(2023)Abel, Steeves, Krsti{\'c}, and Jankovi{\'c}}]{abel2023prescribed}
Abel, I., Steeves, D., Krsti{\'c}, M., and Jankovi{\'c}, M. (2023).
\newblock Prescribed-time safety design for strict-feedback nonlinear systems.
\newblock \emph{IEEE Transactions on Automatic Control}, 69(3), 1464--1479.

\bibitem[{Ames et~al.(2019)Ames, Coogan, Egerstedt, Notomista, Sreenath, and Tabuada}]{ames2019control}
Ames, A.D., Coogan, S., Egerstedt, M., Notomista, G., Sreenath, K., and Tabuada, P. (2019).
\newblock Control barrier functions: Theory and applications.
\newblock In \emph{2019 18th European control conference (ECC)}, 3420--3431. Ieee.

\bibitem[{Li et~al.(2018)Li, Wang, Pierpaoli, and Egerstedt}]{li2018formally}
Li, A., Wang, L., Pierpaoli, P., and Egerstedt, M. (2018).
\newblock Formally correct composition of coordinated behaviors using control barrier certificates.
\newblock In \emph{2018 IEEE/RSJ International Conference on Intelligent Robots and Systems (IROS)}, 3723--3729. IEEE.

\bibitem[{Luo et~al.(2020)Luo, Yi, and Sycara}]{luo2020behavior}
Luo, W., Yi, S., and Sycara, K. (2020).
\newblock Behavior mixing with minimum global and subgroup connectivity maintenance for large-scale multi-robot systems.
\newblock In \emph{2020 IEEE International Conference on Robotics and Automation (ICRA)}, 9845--9851. IEEE.

\bibitem[{{\"O}zkahraman and {\"O}gren(2020)}]{ozkahraman2020combining}
{\"O}zkahraman, {\"O}. and {\"O}gren, P. (2020).
\newblock Combining control barrier functions and behavior trees for multi-agent underwater coverage missions.
\newblock In \emph{2020 59th IEEE Conference on Decision and Control}, 5275--5282. IEEE.

\bibitem[{Pickem et~al.(2017)Pickem, Glotfelter, Wang, Mote, Ames, Feron, and Egerstedt}]{pickem2017robotarium}
Pickem, D., Glotfelter, P., Wang, L., Mote, M., Ames, A., Feron, E., and Egerstedt, M. (2017).
\newblock The robotarium: A remotely accessible swarm robotics research testbed.
\newblock In \emph{2017 IEEE International Conference on Robotics and Automation (ICRA)}, 1699--1706. IEEE.

\bibitem[{Rou{\v{c}}ek et~al.(2019)Rou{\v{c}}ek, Pecka, {\v{C}}{\'\i}{\v{z}}ek, Pet{\v{r}}{\'\i}{\v{c}}ek, Bayer, {\v{S}}alansk{\`y}, He{\v{r}}t, Petrl{\'\i}k, B{\'a}{\v{c}}a, Spurn{\`y} et~al.}]{rouvcek2019darpa}
Rou{\v{c}}ek, T., Pecka, M., {\v{C}}{\'\i}{\v{z}}ek, P., Pet{\v{r}}{\'\i}{\v{c}}ek, T., Bayer, J., {\v{S}}alansk{\`y}, V., He{\v{r}}t, D., Petrl{\'\i}k, M., B{\'a}{\v{c}}a, T., Spurn{\`y}, V., et~al. (2019).
\newblock Darpa subterranean challenge: Multi-robotic exploration of underground environments.
\newblock In \emph{International Conference on Modelling and Simulation for Autonomous Systems}, 274--290. Springer.

\bibitem[{Sabattini et~al.(2013)Sabattini, Secchi, Chopra, and Gasparri}]{sabattini2013distributed}
Sabattini, L., Secchi, C., Chopra, N., and Gasparri, A. (2013).
\newblock Distributed control of multirobot systems with global connectivity maintenance.
\newblock \emph{IEEE Transactions on Robotics}, 29(5), 1326--1332.

\bibitem[{Wang et~al.(2025)Wang, Jiang, Zhu, Song, Du, and Guan}]{wang2025collision}
Wang, C., Jiang, Y., Zhu, S., Song, L., Du, X., and Guan, X. (2025).
\newblock Collision-free connectivity topology switching and maintenance for autonomous underwater vehicles.
\newblock \emph{IEEE Transactions on Control Systems Technology}.

\bibitem[{Yang et~al.(2024{\natexlab{a}})Yang, Lyu, Zhang, Gao, and Luo}]{yang2024integrating}
Yang, Y., Lyu, Y., Zhang, Y., Gao, I., and Luo, W. (2024{\natexlab{a}}).
\newblock Integrating online learning and connectivity maintenance for communication-aware multi-robot coordination.
\newblock In \emph{2024 IEEE/RSJ International Conference on Intelligent Robots and Systems (IROS)}, 5770--5776. IEEE.

\bibitem[{Yang et~al.(2024{\natexlab{b}})Yang, Lyu, Zhang, Yi, and Luo}]{yang24decentralized}
Yang, Y., Lyu, Y., Zhang, Y., Yi, S., and Luo, W. (2024{\natexlab{b}}).
\newblock {Decentralized Multi-Robot Line-of-Sight Connectivity Maintenance under Uncertainty}.
\newblock In \emph{Proceedings of Robotics: Science and Systems}. Delft, Netherlands.
\newblock \doi{10.15607/RSS.2024.XX.005}.

\end{thebibliography}
\end{document}